\documentclass[11pt]{article} 
\usepackage[T1]{fontenc}
 \usepackage[utf8]{inputenc} 
 \usepackage{lmodern} 
 \usepackage{amsmath,amssymb,amsthm}
  \usepackage{graphicx} 
  \usepackage{tikz} 
  \usepackage{geometry}
  
  \usepackage{hyperref}
  \usepackage{authblk}

   \newtheorem{proposition}{Proposition}

    \title{Student Copula Markov Networks} 
    
    \author[1]{Salha Mamane} \author[2]{Cheikh Tidiane Seck} \affil[1]{School of Statistics and Actuarial Science, University of the Witwatersrand, \texttt{Salha.Mamane@wits.ac.za}} \affil[2]{Department of Mathematics, Alioune Diop University of Bambey, \texttt{cheikhtidiane.seck@uadb.edu.sn}}
    
    \date{} 
    
\begin{document}
    	
    	 \maketitle
  \begin{abstract}
  	{This paper extends the applicability of probabilistic graphical models to non Gaussian data.
  		It establishes a characterization of the pairwise Markov property  in undirected probabilistic graphical models in terms of copula densities. Building on this characterization, we derive a necessary condition for the pairwise Markov property in Student copula Markov networks. An Expectation-Conditional Maximization algorithm (ECM) is provided to facilitate the  penalized maximum likelihood estimation and  graph selection. 
  		The performance of the proposed algorithm is evaluated through simulation studies, which demonstrate its effectiveness in  estimating model parameters and recovering the underlying graph structure. 
  	}
  	
\end{abstract}

\maketitle
\section{Introduction}
A probabiiistic graphical model for a random vector  $\mathbf{X}=(X_1,\ldots,X_d)$, $d\ge 2$, is a family of multivariate probability distributions associated to a graph which encodes the conditional independence relations among the variables $X_i$. These conditional independence relations translate into a set of factorization properties that are shared by all members of the family of probability distributions constituting the model. As each factor only includes a subset of the variables, this provides a modular representation of the model, thereby allowing more efficient inference and  computations. Graphical models are broadly classified into two main groups: directed graphical models, also referred to as Bayesian networks 
or belief networks,  which use  directed edges between the nodes in the graph, and undirected graphical models,  also referred to as Markov networks or
Markov random fields, which use undirected edges between the nodes. %in the graph. %representing the statistical model. 

Inference in undirected graphical models is essentially done using  Gaussian distributions, which present a severe limitation in many applications where the marginals are skewed or heavy-tailed.
To weaken the Gaussian assumption, \cite{liu2009,liu2012} used Gaussian copula with nonparametric marginals. They made use of the graphical \textit{lasso} algorithm of \cite{Friedman2008}.
\cite{dobra2011}  used a Bayesian approach  that assumed the observed mixed-type variables are transformations of latent Gaussian variables whose precision matrix encodes the conditional independence structure of an undirected graph.
This essentially corresponds to embedding the Markov network into a Gaussian copula framework. \cite{Hermes2024} extended this approach to heterogeneous, mixed data from different groups.

The Gaussian copula has null tail dependence coefficients, which means  it may be unfitting to applications where the probability of joint extremes events remains high no matter the extreme threshold. 

\cite{finegold2014} explored an approach of graphical modelling with the multivariate $t$-distribution. However, the model is not a Markov network \textit{stricto sensu}. The density does not factor according to the graph, and  the absence of an edge between two nodes does not imply conditional independence of the corresponding variables. Instead, the graph was constructed to  reflect zeros in the precision matrix, a condition that only implies the Markov property in Gaussian models.

The starting point of our methodology is a characterization of the pairwise Markov property in terms of copula densities, which is then applied to common parametric copula families including Gaussian, Student and Archimedean copulas. In particular, in the Student copula case, we obtain a necessary condition of the pairwise Markov property which expresses as  $\rho^2_{ij\cdot V\backslash\{i, j\}}= \alpha(\nu)$, where $\rho_{ij\cdot V\backslash\{i, j\}}$ is the partial correlation coefficient of $X_i$ and $X_j$.  We derive an ECM algorithm for model selection  in $t$-copula graphical models, and evaluate its performance with simulated data.
 
The rest of the paper is organized as follows:
In Section \ref{sec:2}, we derive the characterization of pairwise Markov property in terms of multivariate copula densities, and apply it to common parametric families such as Gaussian, Student and Archimedian. In Section \ref{sec:3}, we present the ECM algorithm for  graph selection in $t$-copula models.  A simulation result is presented in Section \ref{sec:4}.

\section{Characterization of pairwise Markov property through   copula densities}
\label{sec:2}
Let $\textbf{X}=(X_1,\ldots,X_d),\; d\ge 2$, be a continuous random vector with probability density $f(\textbf{x})=f(x_1,\ldots,x_d)$.
The copula function $C$ associated with $\mathbf{X}$ is defined as  the joint cumulative distribution function of the random vector $(F_1(X_1),\ldots,F_d(X_d))$. It gives  a full characterization of the dependence structure of the random variables \cite{Nelsen2006}. If it exists, a copula  density is defined as the derivative, $c$, of the copula distribution function $C$ with respect to the Lebesgue measure :
$$ c(u_1\ldots,u_d)=\frac{\partial^d}{\partial u_1\cdots \partial u_d }C(u_1,\ldots,u_d), \enskip \forall \;(u_1,\ldots,u_d)\in (0,1)^d.$$
$c$ is related to the joint density function $f(\textbf{x})$ by 
\begin{equation}
	f(\textbf{x}) = c(F_1(x_1), \cdots, F_d(x_d)) \prod_{k=1}^df_k(x_k). \label{density} 
\end{equation}

%\textcolor{red} 
For any vector $x=(x_1,\cdots,x_d)\in \mathbb{R}^d$ and   $1\leq i,j\leq d$, 
let us introduce the vectors:	 
	$$x_{-\{i\}}=(x_1,\ldots,x_{i-1},x_{i+1},\ldots,x_d);\qquad x_{-\{j\}}=(x_1,\ldots,x_{j-1},x_{j+1},\ldots,x_d);$$
	$$x_{-\{i,j\}}=(x_1,\ldots,x_{i-1},x_{i+1},\ldots,x_{j-1},x_{j+1},\ldots,x_d),\quad i\le j.$$
	Let $V=\{1,2,\ldots,d\}$ and denote by $X_{V\backslash\{i, j\}}$ the vector defined by all components of $\textbf{X}$; except the components $i$ and $j$. For all $i,j\in V$, $c_{-\{i,j\}}$ represents the copula density associated with the vector $X_{V\backslash\{i,j\}}$;  $c_{-\{i\}}$ represents the copula density associated with the vector $X_{V\backslash\{i\}}$ and $c_{-\{j\}}$ represents the copula density associated with the vector $X_{V\backslash\{j\}}$.\\
Then we have, for all  $u=(u_1\ldots,u_d)\in[0,1]^d$,

	\begin{equation}
	X_{i} \perp X_{j} | X_{V\backslash\{i, j\}}\; \Longleftrightarrow\;	c(u)c_{-\{i,j\}}(u_{-\{i,j\}}) = c_{-\{i\}}(u_{-\{i\}})c_{-\{j\}}(u_{-\{j\}}),\label{main}
	\end{equation}
where $u_{-\{i\}}$, $u_{-\{j\}}$ and $u_{-\{i,j\}}$ are respectively defined analogously as $x_{-\{i\}}$, $x_{-\{j\}}$ and $x_{-\{i,j\}}$ above.
 
The proof of  \eqref{main} is trivial by writing the conditional independence relation in terms of conditional copula densities: For every $(i,j)\in V^2$, 
	\begin{eqnarray}
		X_i\perp X_j | X_{V\backslash\{i,j\} }  &\Leftrightarrow & f(x_i,x_j|x_{-\{i,j\}})=f(x_i|x_{-\{i,j\}})f(x_j|x_{-\{i,j\}}) \nonumber \\
		&\Leftrightarrow & \frac{f(x)}{f(x_{-\{i,j\}})}=	\frac{f(x_{-\{j\}})}{f(x_{-\{i,j\}})}	 \frac{f(x_{-\{i\}})}{f(x_{-\{i,j\}})}\nonumber\\
		&\Leftrightarrow &  f(x)f(x_{-\{i,j\}}) = f(x_{-\{i\}})f(x_{-\{j\}}). \label{s0}
	\end{eqnarray}
	Let $ u = (F_1(x_1), \cdots, F_d(x_d))$, then Statement \eqref{density} yields 
	\begin{equation}
		f(x) = c(u) \prod_{k=1}^d f_k(x_k). 
	\end{equation}
	Analogously, we can write
	\begin{equation}
		f(x_{-\{i\}}) = c_{-\{i\}}(u_{-\{i\}}) \prod_{k\neq i}f_k(x_k), \qquad	f(x_{-\{j\}}) = c_{-\{j\}}(u_{-\{j\}}) \prod_{k\neq j}f_k(x_k),
	\end{equation}
	and 
	\begin{equation}
		f(x_{-\{i,j\}}) = c_{-\{i,j\}}(u_{-\{i,j\}}) \prod_{k\neq i,j}f_k(x_k).
	\end{equation}
	Hence, Condition \eqref{s0} simplifies to
	\begin{equation}\label{denscop}
		c(u)c_{-\{i,j\}}(u_{-\{i,j\}}) = c_{-\{i\}}(u_{-\{i\}})c_{-\{j\}}(u_{-\{j\}}).
	\end{equation}
%\end{proof}
\subsection{Gaussian copula}

The following result is well-known. We provide a direct proof as it will be similar for Student copula.
\begin{proposition}\label{p1}
	Assume the dependence structure of  $\textbf{X}$ follows a $d$-dimensional Gaussian copula with zero mean and correlation matrix $\Sigma$ such that $K= \Sigma^{-1}$ is the precision matrix.  If	$X_{i} \perp X_{j} | X_{V\backslash\{i, j\}}$, then $K_{ij} = 0$.
\end{proposition}
%\textbf{Remark}. Proposition \ref{p1} extends  the classical known result in the Gaussian case; i.e. $X_{i} \perp X_{j} | X_{V\backslash\{i, j\}} \Leftrightarrow K_{ij} = 0$, to the Gaussian copula case. %to follow a multivariate normal distribution 
%Indeed, we show that $K_{ij}=0$ is a necessary condition for the pairwise Markov property as soon as the dependence structure the vector $\textbf{X}$ is modeled by a Gaussian copula and  the marginal distributions being unconstrained.
\begin{proof} The Gaussian copula density is given by
	\begin{equation}
		c(u):=c(u_1,\ldots,u_d ) =|\Sigma|^{-1/2}\exp(-\frac{1}{2}z^{T}(\Sigma^{-1}-I_d)z),
	\end{equation}
	with $z=(\Phi^{-1}(u_1),\ldots,\Phi^{-1}(u_d))^T$, and $I_d$ stand for the identity matrix with $d$ order.  We have
	
	\begin{equation}
		c_{-\{i\}}(u_{-\{i\}})= |\Sigma_{-\{i\}}|^{-1/2}\exp(-\frac{1}{2}z_{-\{i\}}^{T}(\Sigma_{-\{i\}}^{-1}-I_{d-1})z_{-\{i\}}),
	\end{equation}
	
	\begin{equation}
		c_{-\{j\}}(u_{-\{j\}})= |\Sigma_{-\{j\}}|^{-1/2}\exp(-\frac{1}{2}z_{-\{j\}}^{T}(\Sigma_{-\{j\}}^{-1}-I_{d-1})z_{-\{j\}})
	\end{equation}
	and 
	\begin{equation}
		c_{-\{i,j\}}(u_{-\{i,j\}})=|\Sigma_{-\{i,j\}}|^{-1/2}\exp(-\frac{1}{2}z_{-\{i,j\}}^{T}(\Sigma_{-\{i,j\}}^{-1}-I_{d-2})z_{-\{i,j\}}).
	\end{equation}
	
	The pairwise Markov property  $X_i\perp X_j | X_{V\backslash\{i,j\} }$ requires that \eqref{main} be satisfied for all $u\in [0,1]^d$
	and in particular for $u=(0.5,\cdots, 0.5)$ which gives $z=(0,\cdots, 0)$. Therefore, 
	\eqref{main} implies that 
	\begin{equation}
		|\Sigma||\Sigma_{-\{i,j\}}| =  |\Sigma_{-\{i\}}| |\Sigma_{-\{j\}}|.\label{main2}
	\end{equation}
	We can write $\Sigma$, $\Sigma_{-\{i\}}$ and $\Sigma_{-\{j\}}$ as block matrices:
	\begin{equation*}
		\Sigma = \begin{bmatrix}
			\Sigma_{\{i,j\}} & \Sigma_{\{i,j\},-\{i,j\}} \\
			\Sigma_{-\{i,j\},\{i,j\}} & \Sigma_{-\{i,j\}}
		\end{bmatrix},\;
		\Sigma_{-\{i\}} = \begin{bmatrix}
			\Sigma_{jj} & \Sigma_{\{j\},-\{i,j\}} \\
			\Sigma_{-\{i,j\},\{j\}} & \Sigma_{-\{i,j\}}
		\end{bmatrix},
	\end{equation*}
	$$ \Sigma_{-\{j\}} = \begin{bmatrix}
		\Sigma_{ii} & \Sigma_{\{i\},-\{i,j\}} \\
		\Sigma_{-\{i,j\},\{i\}} & \Sigma_{-\{i,j\}}
	\end{bmatrix}. $$
	
	The determinants of these three matrices can be written as a product of the determinant of $\Sigma_{-\{i,j\}}$ and its  Schur complements 
	\cite{abadir2005}.

	\begin{equation}
		|\Sigma| = |\Sigma_{-\{i,j\}}| |	\Sigma_{\{i,j\}} - \Sigma_{\{i,j\},-\{i,j\}}\Sigma_{-\{i,j\}}^{-1} \Sigma_{-\{i,j\},\{i,j\}} |,
	\end{equation}
	\begin{equation}
		|\Sigma_{-\{i\}}| = |\Sigma_{-\{i,j\}}| |	\Sigma_{jj} - \Sigma_{\{j\},-\{i,j\}}\Sigma_{-\{i,j\}}^{-1} \Sigma_{-\{i,j\},\{j\}} |
	\end{equation}
	
	\begin{equation}
		|\Sigma_{-\{j\}}| = |\Sigma_{-\{i,j\}}| |	\Sigma_{ii} - \Sigma_{\{i\},-\{i,j\}}\Sigma_{-\{i,j\}}^{-1} \Sigma_{-\{i,j\},\{i\}} |.
	\end{equation}
	
	Hence, \eqref{main2} becomes
	\begin{eqnarray*}
		&|	\Sigma_{\{i,j\}} - \Sigma_{\{i,j\},-\{i,j\}}\Sigma_{-\{i,j\}}^{-1} \Sigma_{-\{i,j\},\{i,j\}} | = &\\
		&|	\Sigma_{jj} - \Sigma_{\{j\},-\{i,j\}}\Sigma_{-\{i,j\}}^{-1} \Sigma_{-\{i,j\},\{j\}} | |	\Sigma_{ii} - \Sigma_{\{i\},-\{i,j\}}\Sigma_{-\{i,j\}}^{-1} \Sigma_{-\{i,j\},\{i\}} |.&
	\end{eqnarray*}
	
	The  inversion of the partitioned matrices \cite{abadir2005}  gives 
	\begin{equation}
		\Sigma_{\{i,j\}} - \Sigma_{\{i,j\},-\{i,j\}}\Sigma_{-\{i,j\}}^{-1} \Sigma_{-\{i,j\},\{i,j\}}  = [K_{\{i,j\}}]^{-1},
	\end{equation}
	
	\begin{equation}
		\Sigma_{jj} - \Sigma_{\{j\},-\{i,j\}}\Sigma_{-\{i,j\}}^{-1} \Sigma_{-\{i,j\},\{j\}}  =  K_{jj}^{-1}
	\end{equation}
	and \begin{equation}
		\Sigma_{ii} - \Sigma_{\{i\},-\{i,j\}}\Sigma_{-\{i,j\}}^{-1} \Sigma_{-\{i,j\},\{i\}}  = K_{ii}^{-1}.
	\end{equation}
	Noting  that by symmetry
	\begin{equation*}
		K_{\{i,j\}} = \begin{bmatrix}
			K_{ii} & K_{ij} \\
			K_{ij} & K_{jj}
		\end{bmatrix}\; \text{and}\quad
		|[K_{\{i,j\}}]^{-1}|  = \frac{1}{|K_{\{i,j\} }|},
	\end{equation*}

	\eqref{main2} simplifies to
	\begin{equation}
		|[K_{\{i,j\}}]^{-1}|=|K_{ii}^{-1}||K_{ii}^{-1}|,
	\end{equation}
	which is is equivalent to  $K_{ij}=0$.
\end{proof}

\subsection{Student t copula Markov network }
\begin{proposition}
	Assume the dependence structure of  $\textbf{X}$ follows a t-copula with precision matrix $K$ and $\nu$ degrees of freedom.  If	$X_{i} \perp X_{j} | X_{V\backslash\{i, j\}}$, then 
	\begin{equation}\label{Keq}
		K_{ij}^2 = \alpha(\nu) K_{ii}K_{jj}, \quad \text{with }\quad \alpha(\nu) = 1- \frac{\Gamma\left(\frac{\nu+d}{2} -\frac{1}{2}\right)^4}{\Gamma\left(\frac{\nu+d}{2}\right)^2\Gamma\left(\frac{\nu+d}{2}-1 \right)^2 }.
	\end{equation}
Let $\rho_{ij}$ be the partial correlation coefficient of $X_i$ and $X_j$. then \eqref{Keq} is equivalent to 
\begin{equation}
	\rho_{ij}^2 = \alpha(\nu).
\end{equation}

\end{proposition}
\begin{proof}
	The density of the Student  t copula, with correlation matrix $\Sigma$ and $\nu$ degrees of freedom, is given by
	\begin{equation}
		c(u):=c(u_1,\ldots,u_d ) =|\Sigma|^{-1/2}\frac{\Gamma(\frac{\nu+d}{2})}{\Gamma(\frac{\nu}{2})}\left[\frac{\Gamma(\frac{\nu}{2})}{\Gamma(\frac{\nu+1}{2})}\right]^d\frac{\left(1+\frac{1}{\nu}z^{T}\Sigma^{-1}z\right)^{-\frac{\nu+d}{2}}}{\prod_{k=1}^{d}\left(1+\frac{z_k^2}{\nu}\right)^{-\frac{\nu+1}{2}}},
	\end{equation}
	where $z=(z_1,\ldots,z_d)^T$ is the transpose of $z$, with $z_k=T^{-1}_{\nu}(u_k),\, k=1,\ldots,d$ and $T_{\nu}$ is the cumulative distribution function of the standard univariate Student law with $\nu$ degrees of freedom. \\
	We have 
	%	where $z_{-\{i,j\}}$ 
	\begin{equation}
		c_{-\{i\}}(u_{-\{i\}}) =|\Sigma_{-\{i\}}|^{-1/2}\frac{\Gamma(\frac{\nu+d-1}{2})}{\Gamma(\frac{\nu}{2})}\left[\frac{\Gamma(\frac{\nu}{2})}{\Gamma(\frac{\nu+1}{2})}\right]^{d-1}\frac{\left(1+\frac{1}{\nu}z_{-\{i\}}^{T}\Sigma^{-1}_{-\{i\}}z_{-\{i\}}\right)^{-\frac{\nu+d-1}{2}}}{\prod_{k\neq i}\left(1+\frac{z_k^2}{\nu}\right)^{-\frac{\nu+1}{2}}},
	\end{equation}
	
	\begin{equation}
		c_{-\{j\}}(u_{-\{j\}}) =|\Sigma_{-\{j\}}|^{-1/2}\frac{\Gamma(\frac{\nu+d-1}{2})}{\Gamma(\frac{\nu}{2})}\left[\frac{\Gamma(\frac{\nu}{2})}{\Gamma(\frac{\nu+1}{2})}\right]^{d-1}\frac{\left(1+\frac{1}{\nu}z_{-\{j\}}^{T}\Sigma^{-1}_{-\{j\}}z_{-\{j\}}\right)^{-\frac{\nu+d-1}{2}}}{\prod_{k\neq j}\left(1+\frac{z_k^2}{\nu}\right)^{-\frac{\nu+1}{2}}},
	\end{equation}
	
	\begin{equation}
		c_{-\{i,j\}}(u_{-\{i,j\}} =|\Sigma_{-\{i,j\}}|^{-1/2}\frac{\Gamma(\frac{\nu+d-2}{2})}{\Gamma(\frac{\nu}{2})}\left[\frac{\Gamma(\frac{\nu}{2})}{\Gamma(\frac{\nu+1}{2})}\right]^{d-2}\frac{\left(1+\frac{1}{\nu}z_{-\{i,j\}}^{T}\Sigma^{-1}_{-\{i,j\}}z_{-\{i,j\}}\right)^{-\frac{\nu+d-2}{2}}}{\prod_{k\neq i,j}\left(1+\frac{z_k^2}{\nu}\right)^{-\frac{\nu+1}{2}}}.
	\end{equation}
	Then, for $u=(0.5,\cdots, 0.5)$ which gives $z=(0,\cdots, 0)$,  Equation \eqref{main} becomes
	\begin{eqnarray*}
		|\Sigma|^{-1/2}\frac{\Gamma(\frac{\nu+d}{2})}{\Gamma(\frac{\nu}{2})}|\Sigma_{-\{i,j\}}|^{-1/2}\frac{\Gamma(\frac{\nu+d-2}{2})}{\Gamma(\frac{\nu}{2})} &=& |\Sigma_{-\{i\}}|^{-1/2}\frac{\Gamma(\frac{\nu+d-1}{2})}{\Gamma(\frac{\nu}{2})} |\Sigma_{-\{j\}}|^{-1/2}\frac{\Gamma(\frac{\nu+d-1}{2})}{\Gamma(\frac{\nu}{2})}.
	\end{eqnarray*}	
	Simplifying, we obtain
	\begin{equation}\label{ss}
		\left(\frac{|\Sigma|.|\Sigma_{-\{i,j\}}|}{|\Sigma_{-\{i\}}|.|\Sigma_{-\{j\}}|} \right)^{-1/2}= \frac{\Gamma\left(\frac{\nu+d-1}{2} \right)^2}{\Gamma\left(\frac{\nu+d}{2}\right)\Gamma\left(\frac{\nu+d-2}{2} \right)}.	  .
	\end{equation}
	As in the Gaussian case (see, Statements 14--19), we have
	\begin{eqnarray*}
		|\Sigma|&=& |\Sigma_{-\{i,j\}}|.|K_{\{i,j\}}^{-1}|\\
		|\Sigma_{-\{i\}}|&=& |\Sigma_{-\{i,j\}}|.|K_{jj}^{-1}|\\
		|\Sigma_{-\{j\}}|&=& |\Sigma_{-\{i,j\}}|.|K_{ii}^{-1}|.
	\end{eqnarray*}
	and \eqref{ss} yields
	\begin{equation}
		\left(\frac{|K_{\{i,j\}}^{-1}|}{|K_{ii}^{-1}|.|K_{jj}^{-1}|} \right)^{-1/2}= \frac{\Gamma\left(\frac{\nu+d-1}{2} \right)^2}{\Gamma\left(\frac{\nu+d}{2}\right)\Gamma\left(\frac{\nu+d-2}{2} \right)},	  		
	\end{equation}
	with 	$
		K_{\{i,j\}} = \begin{bmatrix}
			K_{ii} & K_{ij} \\
			K_{ij} & K_{jj}
		\end{bmatrix}$,
	$\quad |K_{\{i,j\}}^{-1}|  = |K_{\{i,j\}}|^{-1}=(K_{ii}K_{jj}-K_{ij}^2)^{-1}$, 
	 $\quad |K_{ii}^{-1}|=K_{ii}^{-1}$ ,  $\, |K_{jj}^{-1}|=K_{jj}^{-1}.$
	 
	Therefore,
	\begin{equation}
		\left(\frac{K_{ii}K_{jj}-K_{ij}^2}{K_{ii}.K_{jj}} \right)^{1/2}= \frac{\Gamma\left(\frac{\nu+d-1}{2} \right)^2}{\Gamma\left(\frac{\nu+d}{2}\right)\Gamma\left(\frac{\nu+d-2}{2} \right)},	  		
	\end{equation}
	which is equivalent to
	\begin{equation}
		1-\frac{K_{ij}^2}{K_{ii}.K_{jj}} = \frac{\Gamma\left(\frac{\nu+d-1}{2} \right)^4}{\Gamma\left(\frac{\nu+d}{2}\right)^2\Gamma\left(\frac{\nu+d-2}{2} \right)^2},	  		
	\end{equation}	
	i.e.,
	\begin{equation}
		\frac{K_{ij}^2}{K_{ii}.K_{jj}}=\alpha(\nu), \quad \text{with } \alpha(\nu) = 1- \frac{\Gamma\left(\frac{\nu+d-1}{2}\right)^4}{\Gamma\left(\frac{\nu+d}{2}\right)^2\Gamma\left(\frac{\nu+d-2}{2} \right)^2 }.
	\end{equation}
	
\end{proof}
\subsection{Archimedian copulas}

Copulas models with few real parameters such as the	Archimedian  copulas are incompatible with pairwise Markov properties. We illustrate this with the Clayton copula.

The density of the Clayton copula with parameter $\theta$ is given by 
\begin{equation}
	c_\theta(u_1,\dots,u_d)
	=
	\left( \prod_{k=1}^{d-1} (1 + k\theta) \right)
	\left( \prod_{i=1}^d u_i^{-(1+\theta)} \right)
	\left( \sum_{i=1}^d u_i^{-\theta} - d + 1 \right)^{-(d + 1/\theta)}
\end{equation}
with $\theta > -\frac{1}{d}$ and $\theta \neq 0.$
In the vicinity of $u=(1,\cdots,1)$, Equation \eqref{main}  gives 
\begin{align*}
	&\left( \prod_{k=1}^{d-1} (1 + k\theta) \right)\left( \prod_{k=1}^{d-3} (1 + k\theta) \right)  = \left( \prod_{k=1}^{d-2} (1 + k\theta) \right)^2, \\
	\implies &  1 + (d-1)\theta = 1 + (d-2)\theta,
\end{align*}
which has no solution, except $\theta=0$.

\section{ECM algorithm for graph selection in Student t copula Markov network}
\label{sec:3}
\subsection{Scale-mixture representation of  t-distribution} 
 The Student $t$-copula with correlation matrix $\Sigma$ and $\nu$ degrees of freedom, $C_{\Sigma,\nu}$, is an elliptical copula with common marginal distributions.
 Since all the dependence information is contained in the copula,  the marginal distributions do no affect the estimation of $\Sigma$ and $\nu$ or the corresponding conditional independence graph. Therefore, there is no loss of generality in assuming Student marginals or any other marginals. Let us assume Student marginals. This corresponds to a multivariate t distribution.

The $d$-multivariate $t$-distribution can be represented as a scale-mixture: If $Y \sim t_d(0,\Sigma,\nu)$ the multivariate $t$ distribution, we can write  $Y = X/\sqrt{\tau} $, where  $X\sim N_d(0,\Sigma)$ is normally distributed,  and $\tau \sim \Gamma(\nu/2, \nu/2)$ is a Gamma random variable independent of $Y$ (See for instance \cite{liuRubin1995,finegold2014}).

Since the conditional independence constraint is $K_{ij}^2=\alpha(\nu) K_{ii}K_{jj}$, estimates of the precision matrix $K$ and the degree of freedom $\nu$ can be obtained by maximizing the penalized log-likelihood of the data,
using as penalty $\lambda ||\widetilde{K}||_1$, where $\widetilde{K}$ is a matrix with entries $\widetilde{K}_{ij} = K_{ij}^2-\alpha(\nu) K_{ii}K_{jj}$, and the $L_1$ norm of a matrix is the sum of the absolute values of the matrix entries.

\subsection{Penalized maximum likelihood estimation via the ECM algorithm }
Consider $Y_1,\ldots,Y_n$ a random sample  from a $
d$-dimensional distribution whose associated copula is $C_{\Sigma, \nu}$, and set $Y_i=X_i/\sqrt{\tau_i}$, with $X_i\sim N_d(0,\Sigma)$ and  $\tau_1,\ldots,\tau_n$ is a sequence of latent random variables that are independent and identically Gamma distributed, and independent of the $X_i's$. Let $K=\Sigma^{-1}$ be the precision matrix. Estimates of $K$ and the degrees of freedom $\nu$ will be obtained  via an \textit{ECM} algorithm  (see e.g. \cite{meng1993}).
% the penalized complete log-likelihood of the data.
%using as penalty $\lambda ||\widetilde{K}||_1$, where  $\widetilde{K}_{ij} = K_{ij}^2 - \alpha(\nu)  K_{ii}K_{jj}$, and the $L_1$ norm of a matrix is the sum of the absolute values of its entries.

The complete data log likelihood is
\[ \ell(K,\nu) = \sum_{i=1}^n \Bigl[ \log f(x_i\mid\tau_i,K) + \log f(\tau_i\mid\nu) \Bigr]. \]

Substituting the Gaussian and Gamma densities, and dropping constants gives

\begin{align*}
	\ell(K,\nu) &=  \frac{n}{2}\log|K| -\frac12 \operatorname{tr} \!\left( K\sum_{i=1}^n \tau_i x_ix_i^T \right)\\
	&\qquad + \quad n\frac{\nu}{2}\log\frac{\nu}{2} - n\log\Gamma\!\left(\frac{\nu}{2}\right) + \left(\frac{\nu}{2}-1\right) \sum_{i=1}^n \log\tau_i - \frac{\nu}{2} \sum_{i=1}^n \tau_i.
\end{align*}

The derivative of the complete log likelihood function with respect to $K$ is 
\begin{equation}
	\frac{\partial \ell}{\partial K} = \frac{n}{2}K^{-1} - \frac{1}{2} \sum_{i=1}^{n} \tau_ix_ix_i^T . 
\end{equation}

The derivative of the complete log likelihood function with respect to $\nu$ is 
\begin{equation}
	\frac{\partial \ell}{\partial \nu} = \frac{n}{2} \left[ \log\frac{\nu}{2} +1 -\psi\!\left(\frac{\nu}{2}\right) \right] +\frac12 \sum_{i=1}^n \log\tau_i -\frac12 \sum_{i=1}^n \tau_i , 
\end{equation}

where $\psi$ is the digamma function.

The lasso-type penalty is \[ P(K,\nu) = \lambda \sum_{i<j} \left| \widetilde K_{ij} \right| = \lambda \sum_{i<j} \left| K_{ij}^{2} - \alpha(\nu)K_{ii}K_{jj} \right|. \] 

For \(i\neq j\), both fixed, \[ \frac{\partial P}{\partial K_{ij}} = 2\lambda s_{ij}K_{ij} , \qquad s_{ij} = \operatorname{sign}(\widetilde K_{ij})  = \begin{cases}
	1 & \text{if } \widetilde K_{ij} >0\\
	-1 & \text{if } \widetilde K_{ij} < 0\\
	0	& \text{if } \widetilde K_{ij} =0
\end{cases}.\] 
and for every fixed $k$,
\[ \frac{\partial P}{\partial K_{kk}} = -\lambda \alpha(\nu) \left(\sum_{i< k} s_{ik}K_{ii} + \sum_{j> k} s_{kj}K_{jj} \right)=  -\lambda \alpha(\nu) \sum_{j\neq k} s_{kj}K_{jj} = -\lambda \alpha(\nu) \sum_{j=1}^d s_{kj}K_{jj}. \] 

Therefore, the derivative of the penalty with respect to $K$ is 
\begin{equation}
	\frac{\partial P}{\partial K} = 2\lambda  S \circ K - \lambda \alpha(\nu) \operatorname{Diag}(S\operatorname{diag}(K)),
\end{equation}
where $\circ$ is the Hadamard product, $\operatorname{diag}(A)$ is the vector of diagonal elements of the matrix $A$, and  $\operatorname{Diag}(\mathbf{a})$ is the diagonal matrix with $\mathbf{a}$ as diagonal.

Let $\Lambda(x) = \dfrac{\Gamma\left(x-\tfrac{1}{2}\right)^4}{\Gamma(x)^2\Gamma(x-1)^2}$. Noting that

\[ \Lambda'(x) = 2\Lambda(x) \left[ 2\psi\!\left(x-\frac12\right) -\psi(x) -\psi(x-1) \right], \]
the derivative of the penalty with respect to $\nu$ is 

\begin{equation}
	\frac{\partial P}{\partial \nu} = \lambda\Lambda\!\left(\frac{\nu+d}{2}\right) \left[ 2\psi\!\left(\frac{\nu+d-1}{2}\right) -\psi\!\left(\frac{\nu+d}{2}\right) -\psi\!\left(\frac{\nu+d-2}{2}\right) \right] \sum_{i<j} s_{ij}K_{ii}K_{jj}. 
\end{equation}
The derivatives of the  objective function $Q(K, \nu) = \ell(K, \nu) - P(K, \nu)$ with respect to $K$ and $\nu$ are respectively,

\[\frac{\partial Q}{\partial K} = \frac{\partial \ell}{\partial K} - \frac{\partial P}{\partial K} \] and 
\[ \frac{\partial Q}{\partial \nu} = \frac{\partial \ell}{\partial \nu} - \frac{\partial P}{\partial \nu}.\]

Treating the  latent variables \( \tau_1,\ldots,\tau_n \)  as missing data, we can employ the following \textit{ECM} algorithm  for computing the maximum likelihood estimates of $K$ and $\nu$.
\subsubsection*{\textbf{E-step:}} 
Given the current parameter estimates $K^{(t)}$ and $\nu^{(t)}$, the E-step computes $w_i = E(\tau_i\mid x_i,K^{(t)},\nu^{(t)})$ and $u_i = E(\log\tau_i\mid x_i,K^{(t)},\nu^{(t)}).$

But 
\begin{equation}
	\tau_i | X_i \sim \Gamma\left(\dfrac{\nu+d}{2},    \dfrac{\nu + x_i^T K x_i^T}{2} \right).
\end{equation}
Therefore, $w_i = \dfrac{\nu^{(t)} + d }{\nu^{(t)} + x_i^T K^{(t)} x_i^T}$ and $u_i = \psi\left(\dfrac{\nu^{(t)}+d}{2}\right) - \log\left(\dfrac{\nu^{(t)} + x_i^T K^{(t)} x_i^T}{2}\right)$.

\subsubsection*{\textbf{CM-step}:}
The CM-step maximizes the penalized expected complete-data log-likelihood with respect to \(K\) keeping $\nu$ fixed, and with respect to \(\nu\), keeping $K$ fixed.

\begin{enumerate}
\item  \textbf{Solving for $\nu$}

  $\nu^{(t+1)}$ is the solution of 
\begin{align*}
	h(\nu) &= \frac{n}{2} \left[ \log\frac{\nu}{2} +1 -\psi\!\left(\frac{\nu}{2}\right) \right] +\frac12 \sum_{i=1}^n \log\tau_i -\frac12 \sum_{i=1}^n \tau_i  
	- \lambda\Lambda'\!\left(\frac{\nu+d}{2}\right)  \sum_{i<j} s_{ij}K_{ii}K_{jj} =0.
\end{align*}
We will use Brent's method \cite{brent1973} as implemented in the \textsc{R} function \textsf{uniroot()} to find a root of $h$.

\item \textbf{Solving for $K$}

Directly solving the equation 

\begin{equation}
	\frac{n}{2}K^{-1} - \frac{1}{2} \sum_{i=1}^{n} \tau_ix_ix_i^T  + 2\lambda  S \circ K - \lambda \alpha(\nu) \operatorname{Diag}(S\operatorname{diag}(K)) = 0
\end{equation}
is not feasible. We will use the Alternating Direction Method of Multipliers algorithm (ADMM) (See for example \cite{boyd2011,danaher2014}) to find 	  
$K^{(t+1)} = \operatorname{argmax}_{K} \ell(K,\nu) - P(K,\nu) $, written in ADMM form as 	 
\begin{equation}
	\underset{{K,J}}{\operatorname{argmin}} \quad \frac12 \operatorname{tr} \!\left( K\sum_{i=1}^n \tau_i x_ix_i^T \right) - \frac{n}{2}\log|K| +   \lambda \sum_{i<j} \left| J_{ij}^{2} - \alpha(\nu)J_{ii}J_{jj} \right| \quad \text{s.t } K=J.
\end{equation}
The scaled augmented Lagrangian is
\begin{equation}
	\frac12 \operatorname{tr} \!\left( K\sum_{i=1}^n \tau_i x_ix_i^T \right) - \frac{n}{2}\log|K| 
	+   \lambda \sum_{i<j} \left| J_{ij}^{2} - \alpha(\nu)J_{ii}J_{jj} \right| + \frac{\rho}{2} \operatorname{tr}\bigl[(K-J+U)^T (K-J+U)\bigr].
\end{equation}
The ADMM updates are as follows:

\begin{align*}
	K^{(k+1)} &=  \underset{{K}}{\operatorname{argmin}} \left( \frac12 \operatorname{tr} \!\left( K\sum_{i=1}^n \tau_i x_ix_i^T \right) - \frac{n}{2}\log|K|   + \frac{\rho}{2} \operatorname{tr}\bigl[(K-J^{(k)}+U^{(k)})^T (K-J^{(k)}+U^{(k)})\bigr]   \right)\\
	J^{(k+1)} &= \underset{{J}}{\operatorname{argmin}} \left(  \lambda \sum_{i<j} \left| J_{ij}^{2} - \alpha(\nu)J_{ii}J_{jj} \right| + \frac{\rho}{2} \operatorname{tr}\bigl[(K^{(k+1)}-J+U^{(k)})^T (K^{(k+1)}-J+U^{(k)})\bigr]  \right)\\
	U^{(k+1)} &= U^{(k)} + K^{(k+1)} - J^{(k+1)}
\end{align*}

\begin{enumerate}
	\item \underline{$K^{(k+1)}$ update}:
	
	Setting the derivative of the first objective function of the first equation, gives 
	$$\displaystyle\sum_{i=1}^n \tau_i x_ix_i^T -n K^{-1} +2\rho  (K-J+U) =0$$ which implies
	
	$$K - \dfrac{n}{2\rho } K^{-1}  = (U-J) - \dfrac{1}{2\rho }\displaystyle\sum_{i=1}^n \tau_i x_ix_i^T = Q D_\delta Q^T.$$
	
	Therefore 
	$$Q^TKQ - \dfrac{n}{2\rho } Q^TK^{-1}Q = D_\delta.$$ Let $k_1, k_2,\cdots, k_d$ be the diagonal elements of $Q^TKQ$. They satisfy
	$$k_i - \dfrac{n}{2\rho k_i} = \delta_i \implies  k_i^2 - {\delta_i}k_i - \dfrac{n}{2\rho} =0 \implies k_i = \dfrac{\delta_i + \sqrt{\delta_i^2 +  2n/\rho}}{2}.$$
	
	Hence, \begin{equation}
		K^{(k+1)} = Q \operatorname{Diag}(k_1,\cdots, K_d)Q^T.
	\end{equation}
	
	\item \underline{$J^{(k+1)}$ update:}
	Block coordinate descent \cite{shi2016,wright2015} will be applied.  Off-diagonal and diagonal entries of $J$ will be updated successively while keeping the other fixed. 
	The objective function can be written as 
	\begin{align*}
		g(J) &=  \lambda \sum_{i<j} \left| J_{ij}^{2} - \alpha(\nu)J_{ii}J_{jj} \right| + \frac{\rho}{2} \operatorname{tr}\bigl[(J -B)^T (J-B)\bigr] , \quad \text{where } \, B = K+U\\
		&=  \lambda \sum_{i<j} \left| J_{ij}^{2} - \alpha(\nu)J_{ii}J_{jj} \right| + \frac{\rho}{2}\sum_{i=1}^{d}(J_{ii} - B_{ii})^2 
		+ \rho \sum_{i<j} (J_{ij} - B_{ij})^2.
	\end{align*}
	\begin{itemize}
		\item Off-diagonal coordinate  $J_{ij}$ update with all other entries constant.
		Dropping constants, the objective function is
		\begin{align*}
			g(J_{ij}) &= \lambda  \left| J_{ij}^{2} - a  \right| 
			+ \rho (J_{ij} - B_{ij})^2, \quad \text{where } a=\alpha(\nu)J_{ii}J_{jj}\\
			&= \begin{cases}
				(\lambda +\rho) J_{ij}^{2} - 2\rho B_{ij}J_{ij} - \lambda a +\rho B_{ij}^2 & \text{if } J_{ij}\geq a\\
				(\rho -\lambda) J_{ij}^{2} - 2\rho B_{ij}J_{ij} + \lambda a +\rho B_{ij}^2 & \text{if } J_{ij} < a.\\
			\end{cases}
		\end{align*}
	Derivating this objective function, we obtain the $J_{ij}^{(m+1)}$ update:	
		\begin{equation}
			J_{ij}^{(m+1)} = \begin{cases}
				\dfrac{\rho B_{ij}}{\lambda +\rho}  & \text{if } J_{ij}^2 \geq a\\
				\dfrac{\rho B_{ij}}{\rho - \lambda } & \text{if } J_{ij}^2 < a\\
			\end{cases}
		\end{equation}
		\item Diagonal coordinate $J_{kk}$ update with all other entries constant.
		
		Dropping constants, the objective function is
		\begin{align*}
			g(J_{kk}) &= \lambda \sum_{j=k+1}^{d} \left| J_{kj}^{2} - \alpha(\nu)J_{kk}J_{jj} \right|+  \lambda \sum_{i = 1}^{k-1} \left| J_{ik}^{2} - \alpha(\nu)J_{ii}J_{kk} \right|  + \frac{\rho}{2}(J_{kk} - B_{kk})^2 \\
			&= \lambda \sum_{j \neq k} \left| J_{kj}^{2} - \alpha(\nu)J_{kk}J_{jj} \right|   + \frac{\rho}{2}(J_{kk} - B_{kk})^2
		\end{align*}
		is a convex  function.
		Between any two consecutive breakpoints $t_j = \frac{J_{kj}^2}{\alpha(\nu) J_{jj}}$, it is a quadratic function with derivative 
		\begin{equation}
			g'(J_{kk})= - \lambda\alpha(\nu)  \sum_{j \neq k}s_jJ_{jj} + \rho J_{kk} - \rho B_{kk}, \; \text{where } s_j = \operatorname{sign}(J_{kj}^{2} - \alpha(\nu)J_{kk}J_{jj}).
		\end{equation}
		
		Let $t_j^{(p)}$ is the $p$-th smallest of the breakpoints $t_j$'s. Because of the convexity of $g$, only one of the intervals
		
		$(-\infty,\, t_j^{(1)}), (t_j^{(1)},\, t_j^{(2)}), \cdots, (t_j^{(d-2)},\, t_j^{(d-1)}), (t_j^{(d-1)},\,\infty)$ contains its corresponding root 
		\begin{equation}
			J_{kk}^{(t+1)}= B_{kk} + \frac{\lambda \alpha(\nu)}{\rho}  \sum_{j \neq k}s_jJ_{jj}.
		\end{equation}
		
	\end{itemize}
	\item  \underline{$U^{(k+1)}$ update}: $U^{(k+1)} = U^{(k)} + K^{(k+1)} - J^{(k+1)}$.\\
\end{enumerate} 

\end{enumerate}

\section{Simulation Results}
\label{sec:4}

We simulated two sets of  2000 observations from a 10-dimensional Student copula distribution where the conditional independence relations were encoded in an $A_{10}$ graph:

\begin{tikzpicture}[ every node/.style={circle,draw,minimum size=1mm} ]
	 \node (1) at (0.0,0) [label=above:$X_1$] {}; \node (2) at (1.1,0) [label=above:$X_2$] {}; \node (3) at (2.2,0) [label=above:$X_3$] {}; \node (4) at (3.3,0) [label=above:$X_4$] {}; \node (5) at (4.4,0) [label=above:$X_5$] {}; \node (6) at (5.5,0) [label=above:$X_6$] {}; \node (7) at (6.6,0) [label=above:$X_7$] {}; \node (8) at (7.7,0) [label=above:$X_8$] {}; \node (9) at (8.8,0) [label=above:$X_9$] {}; \node (10) at (9.9,0) [label=above:$X_{10}$] {}; \draw (1) -- (2); \draw (2) -- (3); \draw (3) -- (4); \draw (4) -- (5); \draw (5) -- (6); \draw (6) -- (7); \draw (7) -- (8); \draw (8) -- (9); \draw (9) -- (10); 
\end{tikzpicture}

The first generated data uses Student marginals. The second generated data  uses a mix of symmetric, left-skewed, right-skewed, and  heavy tailed distributions. They include normal, Student, lognormal, gamma, Weibull, chi-squared, beta, and  logistic distributions).

The true parameters are $\nu = 4$  and

\[ K = \begin{pmatrix} 2 & -1.18 & 0.57 & -0.57 & 0.57 & -0.57 & 0.57 & -0.57 & 0.57 & -0.57 \\ -1.18 & 2 & -1.18 & 0.57 & -0.57 & 0.57 & -0.57 & 0.57 & -0.57 & 0.57 \\ 0.57 & -1.18 & 2 & -1.18 & 0.57 & -0.57 & 0.57 & -0.57 & 0.57 & -0.57 \\ -0.57 & 0.57 & -1.18 & 2 & -1.18 & 0.57 & -0.57 & 0.57 & -0.57 & 0.57 \\ 0.57 & -0.57 & 0.57 & -1.18 & 2 & -1.18 & 0.57 & -0.57 & 0.57 & -0.57 \\ -0.57 & 0.57 & -0.57 & 0.57 & -1.18 & 2 & -1.18 & 0.57 & -0.57 & 0.57 \\ 0.57 & -0.57 & 0.57 & -0.57 & 0.57 & -1.18 & 2 & -1.18 & 0.57 & -0.57 \\ -0.57 & 0.57 & -0.57 & 0.57 & -0.57 & 0.57 & -1.18 & 2 & -1.18 & 0.57 \\ 0.57 & -0.57 & 0.57 & -0.57 & 0.57 & -0.57 & 0.57 & -1.18 & 2 & -1.18 \\ -0.57 & 0.57 & -0.57 & 0.57 & -0.57 & 0.57 & -0.57 & 0.57 & -1.18 & 2 \end{pmatrix}. \]

Thus, $\alpha(\nu) \approx 0.0799$, $\sqrt{\alpha(\nu)} \approx  0.2827$,  and the true partial correlation matrix is 

\[\boldsymbol{\rho} \approx  \begin{pmatrix} 1 & -0.59 & 0.28 & -0.28 & 0.28 & -0.28 & 0.28 & -0.28 & 0.28 & -0.28 \\ -0.59 & 1 & -0.59 & 0.28 & -0.28 & 0.28 & -0.28 & 0.28 & -0.28 & 0.28 \\ 0.28 & -0.59 & 1 & -0.59 & 0.28 & -0.28 & 0.28 & -0.28 & 0.28 & -0.28 \\ -0.28 & 0.28 & -0.59 & 1 & -0.59 & 0.28 & -0.28 & 0.28 & -0.28 & 0.28 \\ 0.28 & -0.28 & 0.28 & -0.59 & 1 & -0.59 & 0.28 & -0.28 & 0.28 & -0.28 \\ -0.28 & 0.28 & -0.28 & 0.28 & -0.59 & 1 & -0.59 & 0.28 & -0.28 & 0.28 \\ 0.28 & -0.28 & 0.28 & -0.28 & 0.28 & -0.59 & 1 & -0.59 & 0.28 & -0.28 \\ -0.28 & 0.28 & -0.28 & 0.28 & -0.28 & 0.28 & -0.59 & 1 & -0.59 & 0.28 \\ 0.28 & -0.28 & 0.28 & -0.28 & 0.28 & -0.28 & 0.28 & -0.59 & 1 & -0.59 \\ -0.28 & 0.28 & -0.28 & 0.28 & -0.28 & 0.28 & -0.28 & 0.28 & -0.59 & 1 \end{pmatrix} .\]

We implemented 200 iterations of the ECM algorithm, with initial values $\nu=1$ and $K = I$,  the identity matrix, and a convergence breaking when neither $\nu$ nor any element of $K$ changed by more than $10^{-6}$ from one iteration to the next.

\subsection*{Results with first dataset:}

The algorithm converged after just 24 iterations with $\widehat{\nu} = 4.22$ and 

\[\boldsymbol{\widehat{\rho}} = \begin{pmatrix} 1 & -0.58 & 0.25 & -0.23 & 0.25 & -0.27 & 0.25 & -0.25 & 0.27 & -0.29 \\ -0.58 & 1 & -0.58 & 0.28 & -0.28 & 0.30 & -0.29 & 0.29 & -0.30 & 0.31 \\ 0.25 & -0.58 & 1 & -0.59 & 0.28 & -0.28 & 0.27 & -0.28 & 0.30 & -0.28 \\ -0.23 & 0.28 & -0.59 & 1 & -0.59 & 0.28 & -0.27 & 0.27 & -0.29 & 0.28 \\ 0.25 & -0.28 & 0.28 & -0.59 & 1 & -0.58 & 0.26 & -0.28 & 0.29 & -0.28 \\ -0.27 & 0.30 & -0.28 & 0.28 & -0.58 & 1 & -0.58 & 0.24 & -0.25 & 0.27 \\ 0.25 & -0.29 & 0.27 & -0.27 & 0.26 & -0.58 & 1 & -0.55 & 0.25 & -0.25 \\ -0.25 & 0.29 & -0.28 & 0.27 & -0.28 & 0.24 & -0.55 & 1 & -0.59 & 0.25 \\ 0.27 & -0.30 & 0.30 & -0.29 & 0.29 & -0.25 & 0.25 & -0.59 & 1 & -0.58 \\ -0.29 & 0.31 & -0.28 & 0.28 & -0.28 & 0.27 & -0.25 & 0.25 & -0.58 & 1 \end{pmatrix}. \]
Thus, $\sqrt{\widehat{\alpha(\nu)}} \approx  0.28$.

The relative $L_1$ error is $\frac{||\boldsymbol{\widehat{\rho}}-\boldsymbol{{\rho}}||_1}{||\boldsymbol{{\rho}}||_1} = 5\%$.

The maximum discrepancy between $\boldsymbol{\widehat{\rho}}_{ij}$ and $\sqrt{\widehat{\alpha(\nu)}}$ is $18\%$, the mean discrepancy is $5.6\%$, and the median discrepancy is $4.7\%$.

Allowing a maximumun discrepancy of  $20\%$, the true graph is recovered perfectly. 

For comparison, with the glasso algorithm \cite{Friedman2008}, we obtained

\[ \boldsymbol{\widehat{\rho}} =  \begin{pmatrix} 1 & -0.37 & 0 & 0 & 0 & 0 & 0 & 0 & 0 & 0 \\ -0.37 & 1 & -0.34 & 0 & 0 & 0 & 0 & 0 & 0 & 0 \\ 0 & -0.34 & 1 & -0.37 & 0 & -0.01 & 0 & 0 & 0 & 0 \\ 0 & 0 & -0.37 & 1 & -0.37 & 0 & 0 & 0 & 0 & 0 \\ 0 & 0 & 0 & -0.37 & 1 & -0.35 & 0 & -0.03 & 0 & 0 \\ 0 & 0 & -0.01 & 0 & -0.35 & 1 & -0.37 & 0 & 0 & 0 \\ 0 & 0 & 0 & 0 & 0 & -0.37 & 1 & -0.33 & 0 & -0.02 \\ 0 & 0 & 0 & 0 & -0.03 & 0 & -0.33 & 1 & -0.37 & 0 \\ 0 & 0 & 0 & 0 & 0 & 0 & 0 & -0.37 & 1 & -0.37 \\ 0 & 0 & 0 & 0 & 0 & 0 & -0.02 & 0 & -0.37 & 1 \end{pmatrix} \]
 corresponding to a relative $L_1$ error of $55\%$.

\subsection*{Results with second dataset:}

The algorithm did not converge with the breaking rule $10^{-6}$, reducing this requirement to $10^{-4}$, the algorithm converged after $36$ iterations, yielding estimates  $\widehat{\nu} = 3.47$  and 

\[ \boldsymbol{\widehat{\rho}} = \begin{pmatrix} 1 & -0.53 & 0.17 & -0.13 & 0.19 & -0.22 & 0.19 & -0.19 & 0.11 & -0.17 \\ -0.53 & 1 & -0.46 & 0.10 & -0.20 & 0.24 & -0.26 & 0.29 & 0.12 & 0.11 \\ 0.17 & -0.46 & 1 & -0.46 & 0.21 & -0.24 & 0.24 & -0.29 & -0.10 & -0.07 \\ -0.13 & 0.10 & -0.46 & 1 & -0.49 & 0.16 & -0.10 & 0.01 & -0.44 & 0.23 \\ 0.19 & -0.20 & 0.21 & -0.49 & 1 & -0.53 & 0.22 & -0.29 & 0.23 & -0.14 \\ -0.22 & 0.24 & -0.24 & 0.16 & -0.53 & 1 & -0.58 & 0.29 & -0.25 & 0.15 \\ 0.19 & -0.26 & 0.24 & -0.10 & 0.22 & -0.58 & 1 & -0.51 & -0.06 & -0.12 \\ -0.19 & 0.29 & -0.29 & 0.01 & -0.29 & 0.29 & -0.51 & 1 & -0.29 & 0 \\ 0.11 & 0.12 & -0.10 & -0.44 & 0.23 & -0.25 & -0.06 & -0.29 & 1 & -0.15 \\ -0.17 & 0.11 & -0.07 & 0.23 & -0.14 & 0.15 & -0.12 & 0 & -0.15 & 1 \end{pmatrix} \]

Thus, $\sqrt{\widehat{\alpha(\nu)}} \approx  0.29$.  

The relative $L_1$ error is $\frac{||\boldsymbol{\widehat{\rho}}-\boldsymbol{{\rho}}||_1}{||\boldsymbol{{\rho}}||_1} = 24\%$.

The maximum discrepancy between $\boldsymbol{\widehat{\rho}}_{ij}$ and $\sqrt{\widehat{\alpha(\nu)}}$ was $99\%$, the average discrepancy was $40\%$ and the median discrepancy was $39\%$.
Allowing a discrepancy of $20\%$, $28\%$ percent of the $36$ missing edges were correctly identified, while only one true edge ($11\%$) was not recovered. Raising the allowed discrepancy to $50\%$, $61\%$ of the missing edges were correctly identified, while all true edges were perfectly recovered.

For comparison, with the glasso algorithm, we obtained

\[\boldsymbol{\widehat{\rho}} = \begin{pmatrix} 1 & -0.32 & 0 & 0 & 0 & 0 & 0 & 0 & 0 & 0 \\ -0.32 & 1 & -0.28 & 0 & 0 & 0 & 0 & 0 & 0 & 0 \\ 0 & -0.28 & 1 & -0.38 & 0 & -0.03 & 0 & 0 & 0 & 0 \\ 0 & 0 & -0.38 & 1 & -0.38 & 0 & -0.02 & 0 & 0 & 0 \\ 0 & 0 & 0 & -0.38 & 1 & -0.27 & -0.05 & 0 & 0 & 0 \\ 0 & 0 & -0.03 & 0 & -0.27 & 1 & -0.37 & 0 & 0 & 0 \\ 0 & 0 & 0 & -0.02 & -0.05 & -0.37 & 1 & 0 & 0 & -0.10 \\ 0 & 0 & 0 & 0 & 0 & 0 & 0 & 1 & 0 & 0 \\ 0 & 0 & 0 & 0 & 0 & 0 & 0 & 0 & 1 & 0 \\ 0 & 0 & 0 & 0 & 0 & 0 & -0.10 & 0 & 0 & 1 \end{pmatrix},\] 
corresponding to a relative $L_1$ error of $ 64\%$.


\begin{thebibliography}{99}
	
	
	\bibitem{abadir2005}
	Abadir, K.~M.~ and Magnus, J.~R.~ (2005).
	\textit{Matrix Algebra}.
	Cambridge University Press.
	
%	\bibitem{bauer2016}
%	Bauer, A. and Czado, C. (2016).
%	Pair-Copula Bayesian Networks.
%	\textit{Journal of Computational and Graphical Statistics}, 25(4), 1248--1271.
	
	%\bibitem[Bouezmarni et al.(2012)]{bouezmarni2012}Bouezmarni, T., Rombouts, J.V.K. and Taamouti, A. (2012). Nonparametric copula-based test for conditional independence with applications to Granger causality.\textit{Journal of Business \& Economic Statistics}, 30(2), 275--287.
	
	\bibitem{boyd2011}
	Boyd, S., Parikh, N., Chu, E., Peleato, B. and Eckstein, J. (2011).
	Distributed optimization and statistical learning via the alternating
	direction method of multipliers.
	\textit{Foundations and Trends in Machine Learning}, 3(1), 1--122.
	doi:10.1561/2200000016.
	
\bibitem{brent1973} Brent, R. P. (1973). \textit{Algorithms for Minimization without Derivatives}. Englewood Cliffs, NJ: Prentice-Hall.
	
	\bibitem{danaher2014}
	Danaher, P., Wang, P. and Witten, D. M. (2014).
	The joint graphical lasso for inverse covariance estimation across
	multiple classes.
	\textit{Journal of the Royal Statistical Society: Series B
		(Statistical Methodology)}, 76(2), 373--397.
	doi:10.1111/rssb.12033.
	
	\bibitem{dobra2011}
	Dobra, A. and Lenkoski, A. (2011).
	Copula Gaussian graphical models and their application to modeling functional disability data.
	\textit{The Annals of Applied Statistics}, 5(2A), 969--993.
	
%	\bibitem{elidan2010}
%	Elidan, G. (2010).
%	Copula Bayesian networks.
%	\textit{Advances in Neural Information Processing Systems}, 23.
%	
%	\bibitem{elidan2013}
%	Elidan, G. (2013).
%	Copulas in machine learning.
%	In \textit{Copulae in Mathematical and Quantitative Finance: Proceedings of the Workshop Held in Cracow, 10--11 July 2012}, pp.~39--60.
%	Springer.
	
	\bibitem{finegold2014}
	Finegold, M. A. and Drton, M. (2014).
	Robust graphical modeling with $t$-distributions.
	\textit{arXiv preprint arXiv:1408.2033}.
	
%	\bibitem{Frahm2005} Frahm, G., Junker, M., and Schmidt, R. (2005). Estimating the Tail-Dependence Coefficient: Properties and Pitfalls. \textit{Insurance: Mathematics and Economics}, \textbf{37}(1), 80--100.
%	
	\bibitem{Friedman2008} Friedman, J., Hastie, T., and Tibshirani, R. (2008). Sparse Inverse Covariance Estimation with the Graphical Lasso.\textit{Biostatistics}, \textbf{9}(3), 432--441. \url{https://doi.org/10.1093/biostatistics/kxm045}
	
	\bibitem{Hermes2024} Hermes, S., van Heerwaarden, J., and Behrouzi, P. (2024). ``Copula Graphical Models for Heterogeneous Mixed Data.'' \textit{Journal of Computational and Graphical Statistics}, \textbf{33}(3), 991--1005. \url{https://doi.org/10.1080/10618600.2023.2289545}
	 
	
	%\bibitem[Farnoudkia and Purut\c{c}uo\u{g}lu(2021)]{farnoudkia2021}
	%Farnoudkia, H. and Purut\c{c}uo\u{g}lu, V. (2021).
	%Vine copula graphical models in the construction of biological networks.
	%\textit{Hacettepe Journal of Mathematics and Statistics}, 50(4), 1172--1184.
	
%	\bibitem{lasserre2022}
%	Lasserre, M. (2022).
%	\textit{Apprentissages dans les réseaux bayésiens à base de copules non-paramétriques}.
%	PhD thesis, Sorbonne Université.
%	
%	\bibitem{lauritzen1989}
%	Lauritzen, S.L. and Wermuth, N. (1989).
%	Graphical models for associations between variables, some of which are qualitative and some quantitative.
%	\textit{The Annals of Statistics}, 17(1), 31--57.
	
	\bibitem{liu2009}
	Liu, H., Lafferty, J. and Wasserman, L. (2009).
	The nonparanormal: Semiparametric estimation of high-dimensional undirected graphs.
	\textit{Journal of Machine Learning Research}, 10, 2295--2328.
	
	\bibitem{liu2012}
	Liu, H., Han, F., Yuan, M., Lafferty, J. and Wasserman, L. (2012).
	High-dimensional semiparametric Gaussian copula graphical models.
	\textit{The Annals of Statistics}, 40(4), 2293--2326.
	
	\bibitem{liuRubin1995}
	Liu, C. and Rubin, D. B. (1995).
	ML estimation of the $t$ distribution using EM and its extensions, ECM and ECME.
	\textit{Statistica Sinica}, 5(1), 19--39.
	
	\bibitem{meng1993}
	Meng, X.-L. and Rubin, D. B. (1993).
	Maximum likelihood estimation via the ECM algorithm: A general framework.
	\textit{Biometrika}, 80(2), 267--278.
	
	
	\bibitem{Nelsen2006} Nelsen, R. B. (2006). \textit{An Introduction to Copulas}. New York, NY: Springer.
	
	%\bibitem[Pircalabelu et al.(2015)]{pircalabelu2015}
	%Pircalabelu, E., Claeskens, G. and Waldorp, L.J. (2015).
	%A focused information criterion for graphical models.
	%\textit{Statistics and Computing}, 25(6), 1071--1092.
	
	%\bibitem[Sancetta and Satchell(2004)]{sancetta2004}	Sancetta, A. and Satchell, S. (2004).The Bernstein copula and its applications to modeling and approximations of multivariate distributions.	\textit{Econometric Theory}, 20(3), 535--562.
	
	\bibitem{shi2016}
	Shi, H.-J. M., Tu, S., Xu, Y. and Yin, W. (2016).
	A primer on coordinate descent algorithms.
	\textit{arXiv preprint arXiv:1610.00040}.
	
	\bibitem{wright2015}
	Wright, S. J. (2015).
	Coordinate descent algorithms.
	\textit{Mathematical Programming}, 151, 3--34.
	doi:10.1007/s10107-015-0892-3.
\end{thebibliography}
\end{document}